\documentclass[11pt]{article}

\usepackage{acl}

\usepackage{times}
\usepackage{latexsym}

\usepackage[T1]{fontenc}
\usepackage{xcolor}
\usepackage{colortbl}
\usepackage{amsmath}
\usepackage{amssymb}
\usepackage{booktabs}
\usepackage[utf8]{inputenc}
\usepackage{multirow}

\usepackage{microtype}
  \usepackage{amsthm}
  \newtheorem{theorem}{Theorem}[section]
  \newtheorem{proposition}[theorem]{Proposition}
  \newtheorem{lemma}[theorem]{Lemma}
  \newtheorem{corollary}[theorem]{Corollary}
  \theoremstyle{definition}
  \newtheorem{assumption}[theorem]{Assumption}
  \newtheorem{remark}[theorem]{Remark}

\usepackage{inconsolata}

\usepackage{graphicx}

\title{SPARK: Learning from Verification in Speculative GRPO Rollouts}

\author{
\textbf{Nhat Minh Pham}\textsuperscript{1,2,*} \quad
\textbf{Duy Tung Doan}\textsuperscript{2,*} \quad
\textbf{Quang Huy Phung}\textsuperscript{2} \quad
\textbf{Thi Duyen Ngo}\textsuperscript{1}
\\
\textbf{Vinh Van Nguyen}\textsuperscript{1} \quad
\textbf{Khac-Hoai Nam Bui}\textsuperscript{2}
\\[2mm]
\textsuperscript{1}VNU University of Engineering and Technology,
Vietnam National University, Hanoi, Vietnam
\\
\textsuperscript{2}Viettel AI, Viettel Group, Hanoi, Vietnam
\\
\texttt{\{minhpn.266, duytunghanam\}@gmail.com}
\\
\textsuperscript{*}Equal contribution
}

\begin{document}
\maketitle
\begin{abstract}
Group Relative Policy Optimization (GRPO) has become a standard approach for improving the reasoning abilities of large language models, but its training is costly because each step requires generating many long reasoning responses. Recent methods use speculative decoding to accelerate the rollouts, where a draft model proposes tokens that the target model verifies. However, we find that the draft model repeatedly misses the same tokens within a response and across responses in the same GRPO group, and the verification feedback on these misses is not used to update the draft during the rollout. Motivated by this analysis, we introduce SPARK, a lightweight self-adaptive mechanism designed to enhance draft models used in speculative GRPO training. SPARK leverages target verification feedback to refine draft predictions during generation while preserving their existing training procedures. Averaged across 21 model--dataset settings, SPARK improves generation speedup by 32.7\% and 11.4\%, and end-to-end GRPO training speedup by 32.2\% and 8.5\%, over EAGLE3-style and DFlash-style baselines, respectively.
\end{abstract}

\section{Introduction}

\begin{figure}[t]
\centering
\includegraphics[width=\columnwidth]{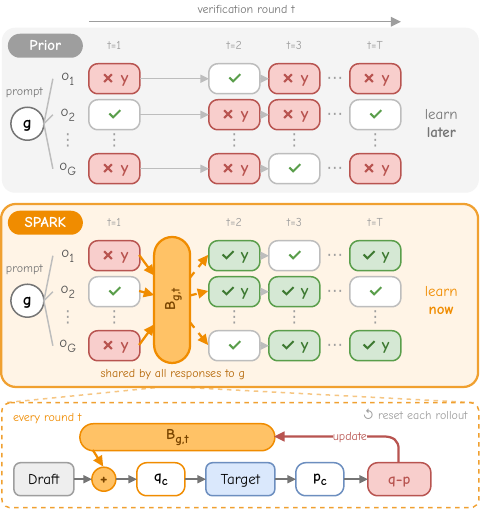}
\caption{\textbf{Learning from verification during the rollout.} In GRPO, we find that the draft model often misses the same token $y$ again within a response and across responses to the same prompt $g$. Prior methods learn from these misses only after the rollout, so they keep repeating. SPARK turns each verification round into a small correction $B_{g,t}$ shared by all responses to $g$, so a miss in one round helps every response in the next.}
\label{fig:spark_teaser}
\end{figure}

Group Relative Policy Optimization (GRPO) has become a standard approach for improving the reasoning abilities of large language models~\citep{shao2024deepseekmath}, but rollout generation dominates its training time: each step samples a group of long responses for every prompt.
Speculative decoding reduces this cost by letting a lightweight draft model propose tokens that the target model verifies in parallel, preserving its sampling distribution.

Speculative RL systems obtain draft proposals in two ways.
Non-parametric methods reuse generated text, either earlier trajectories of the same prompt~\citep{liu2026specrl} or a suffix tree over recent rollouts~\citep{shao2025das}, so they help only when text repeats verbatim.
Learned drafters generalize beyond exact repeats and are trained online to track the changing policy~\citep{zhang2025fastgrpo,chen2026respec,hu2026tlt,iso2026rlacceleration,wang2026onlinedraft}.
However, these learned drafters are primarily improved through parameter updates rather than adjustments made during the ongoing rollout.

\begin{figure*}[t]
    \centering
    
\includegraphics[width=\textwidth]{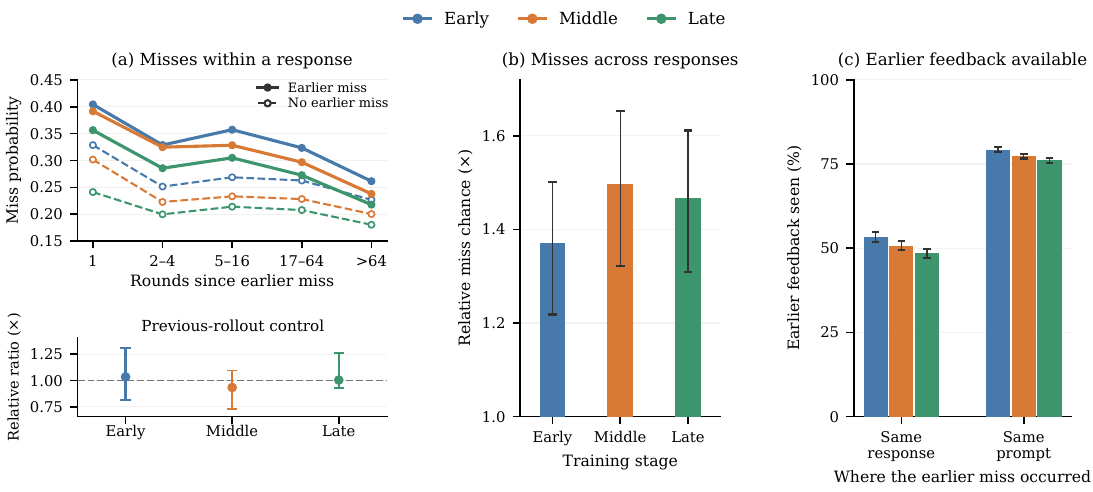}
    \caption{\textbf{Recurring draft misses during GRPO rollouts.}
    (a) The upper plot compares miss probability when a token was missed earlier in the same response and when it was not. The lower plot measures whether an earlier miss in an independent previous rollout predicts another miss in the current rollout; values near $1\times$ indicate little difference.
(b) Relative miss probability after an earlier miss on the same token in another response generated from the same prompt.
(c) Percentage of misses preceded by an earlier miss on the same token within the response or within the prompt group. }
    \label{fig:motivation_analysis}
\end{figure*}

Within a rollout, they ignore the feedback that verification produces every round (Figure~\ref{fig:spark_teaser}, top): beyond deciding which proposals to accept, verification reveals where the draft distribution departs from the target distribution. Whether this information helps later predictions is not obvious: some tokens may simply be difficult for the draft regardless of the rollout. We therefore study draft misses to determine whether earlier verification feedback predicts later errors. As shown in Figure~\ref{fig:motivation_analysis} and examined in Section~\ref{sec:analysis}, misses recur within a response, while misses from an independent previous rollout show little association. We also find that misses are related across responses generated from the same prompt, and that many later misses involve tokens already missed earlier during generation.

We introduce SPARK (Figure~\ref{fig:spark_teaser}, middle and bottom), which allows draft predictions to adapt using verification feedback during the rollout. SPARK adds a low-rank correction to the draft logits, shared within each prompt group and updated after each verification round using target probabilities computed by the verifier. The correction is reset at the start of each rollout, while a projection from draft features is retained and trained with the draft model across rollouts. SPARK requires no additional target forward passes or draft backbone parameter updates during generation, and preserves the verification and draft training procedures.

We evaluate SPARK across models, datasets,
and speculative drafting methods. Across 21 model--dataset settings,
SPARK achieves average generation speedups of $2.43\times$ and $2.30\times$
and average end-to-end GRPO training speedups of $2.31\times$ and $2.11\times$
with EAGLE3- and DFlash-style drafting, respectively. Relative to their
corresponding baselines, SPARK improves generation speedup by 32.7\% and 11.4\%
and end-to-end speedup by 32.2\% and 8.5\%, respectively. SPARK retains the
exact target-distribution sampling guarantee.

Our contributions are:
\begin{itemize}
    \item \textbf{Analysis.} We analyze draft misses during GRPO rollouts and show that earlier verification outcomes are associated with later errors, while an independent previous-rollout control shows little comparable association. This reveals opportunities to use verification feedback before generation ends.

    \item \textbf{Method.} We introduce SPARK, a lightweight adaptation mechanism for draft models used in speculative GRPO training. SPARK leverages target verification feedback to refine draft predictions during generation, without additional target model forward passes or draft backbone parameter updates, while preserving existing online draft training procedures.

    \item \textbf{Evidence.} We evaluate draft acceptance, generation speedup, and end-to-end GRPO training efficiency across different models, datasets, and drafting methods. Ablations examine feedback timing, sharing scope, feedback selection, and projection learning. SPARK improves draft acceptance and training efficiency (Table~\ref{tab:main_results}), with ablations supporting its design choices (Table~\ref{tab:spark_ablations}).
\end{itemize}
\section{Analyzing Verification Feedback During GRPO Rollouts}
\label{sec:analysis}

Each speculative verification round provides target probabilities and reveals where the draft predictions disagree with the target. However, this feedback is not typically used to improve subsequent draft proposals during the same rollout. We therefore examine draft misses to understand whether verification provides information that could be useful before generation ends. A miss occurs when the draft proposals fail to include the token selected by the target. We investigate whether the same mistakes reappear within a response, across responses, and whether earlier feedback is available before later misses occur.
Our experiments use Qwen2.5-3B-Instruct trained with GRPO on SimpleLR~\citep{zeng2025simplerlzoo}\footnote{\url{https://huggingface.co/datasets/Feierabend/simplelr_abel_level3to5}}, with a single-layer EAGLE3-style drafter. We evaluate draft misses at early, middle, and late training stages using a consistent rollout sampling setup. Figure~\ref{fig:motivation_analysis} summarizes the results.

\subsection{Misses Recur Within a Response}
\label{sec:within_response_misses}

We first ask whether the draft is more likely to miss a token again after missing it earlier in the same response. We compare the probability of a miss when the token was previously missed with the probability when it was not, at different distances between verification rounds.

Figure~\ref{fig:motivation_analysis}(a, top) shows that tokens missed earlier have a consistently higher chance of being missed again, even after more than 64 rounds. One possible explanation is that some tokens are simply difficult for the draft, regardless of the rollout. To examine this, we also use misses recorded in an independent previous rollout. In Figure~\ref{fig:motivation_analysis}(a, bottom), we divide the miss probability for tokens missed in the previous rollout by the probability for comparable tokens without such an earlier miss. A ratio of $1\times$ means that both groups have the same miss probability. The measured ratios are $1.03\times$, $0.93\times$, and $1.00\times$ at the early, middle, and late stages. Unlike the clear recurrence within a response, misses from a previous rollout do not consistently predict more misses in the current one. This suggests that feedback from the ongoing rollout is especially relevant to later draft predictions.

\subsection{Misses Are Shared Within a Prompt Group}
\label{b_t_share}

We next examine whether a miss in one response is related to misses in other responses generated from the same prompt. We identify repeated errors by matching a single target token across responses. As above, we measure the ratio between the probability of a miss when the same token was missed earlier in another response and the probability when it was not.

Figure~\ref{fig:motivation_analysis}(b) shows ratios of $1.37\times$, $1.50\times$, and $1.47\times$ at the early, middle, and late stages. In other words, if a token has already been missed in one response, it becomes substantially more likely to be missed again in another response to the same prompt. This result motivates sharing verification feedback among responses generated from that prompt.

\subsection{Earlier Feedback Is Available During Generation}
\label{sec:feedback_availability}

Finally, for earlier feedback to help later draft predictions, it must be available before the next miss occurs. For each miss, we check whether the same token was already missed earlier in the current response or in another response generated from the same prompt.

Figure~\ref{fig:motivation_analysis}(c) shows that, at the late training stage, $48.5\%$ of misses had an earlier occurrence within the same response, and $76.2\%$ had an earlier occurrence within the same prompt. Similar patterns appear at the early and middle stages. Together with the results above, these findings show that verification often provides feedback before similar draft errors occur again, motivating the use of that information during generation rather than only in subsequent draft training.

\section{SPARK: Learning from Verification Feedback During the Rollout}
\label{method}

\begin{figure}[t]
\centering
\includegraphics[width=\columnwidth]{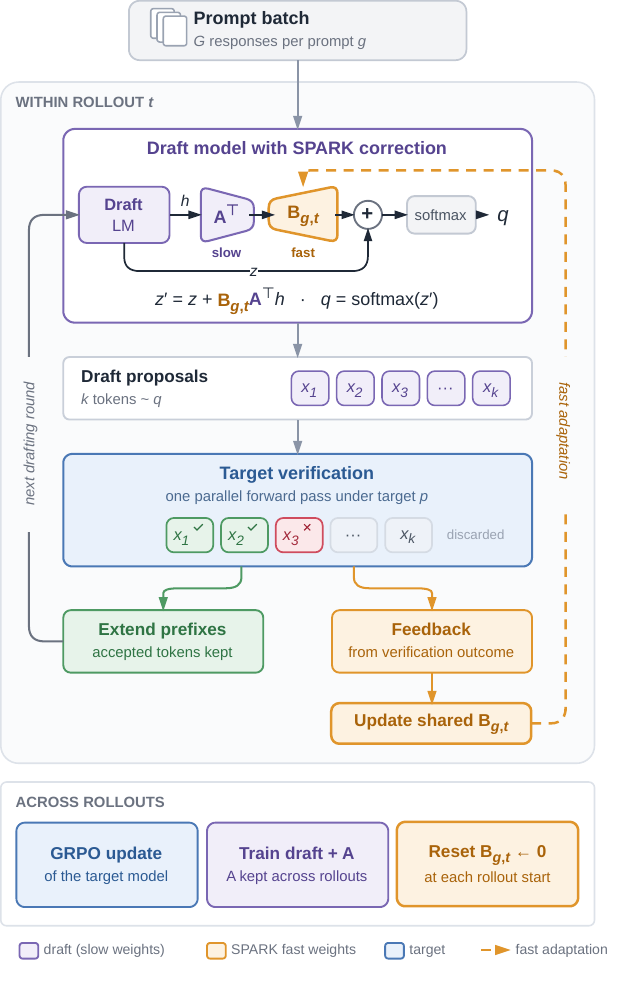}
\caption{\textbf{Overview of SPARK.} During a rollout, the draft logits are corrected by $B_{g,t}A^{\top}h_c$, where $A$ is a projection learned across rollouts and $B_{g,t}$ is a correction matrix shared by the responses to prompt $g$. After each verification round, target distributions already computed by the verifier update $B_{g,t}$. At the next rollout, $B_{g,t}$ is reset to zero, while $A$ is retained and trained with the draft model.}
\label{overview}
\end{figure}

Section~\ref{sec:analysis} shows that draft misses recur within a response and across responses to the same prompt, but not across rollouts. The useful feedback is therefore produced during the rollout and must be used before it ends. SPARK (Figure~\ref{overview}) adds a low-rank correction to the draft logits that is updated after every verification round and reset at the end of the rollout, on top of a projection learned across rollouts. Formal statements and proofs of its properties are given in Appendix~\ref{app:theory}.

\subsection{Preliminaries and Notation}
\label{sec:spark_setup}

A \emph{rollout} is the generation phase of one GRPO step, in which a group of responses is sampled for each prompt. Generation proceeds in verification \emph{rounds}: in round $t$, the drafter proposes tokens for each active response, and the target scores them in one forward pass. A \emph{context} $c$ is a position where the drafter predicts a next token, and $g(c)$ is the prompt of its response. At $c$, $h_c\in\mathbb{R}^{d}$ is the feature entering the draft output head, $z_c\in\mathbb{R}^{V}$ are the draft logits, and $p_c$ is the target distribution used by verification.

\subsection{A Rollout-Local Logit Correction}
\label{submethod1}

Motivated by the shared miss patterns across responses to the same prompt (\S\ref{b_t_share}), SPARK maintains a prompt-specific correction matrix $B_{g,t}\in\mathbb{R}^{V\times r}$ and a projection $A\in\mathbb{R}^{d\times r}$ shared across prompts, where $r\ll d$. At round $t$, the draft logits are adjusted as
\begin{equation}
    \widetilde{z}_c=z_c+B_{g(c),t}\,u_c,\qquad u_c=A^{\top}h_c,
    \label{eq:spark_correction}
\end{equation}
and proposals are sampled from $q_c=\operatorname{softmax}(\widetilde{z}_c)$; proposal construction and verification otherwise remain unchanged. The projected features $u_c$ make the correction context-dependent even though $B_{g,t}$ is shared by the prompt (Proposition~\ref{prop:kernel}), while the low-rank parameterization limits computation and memory. Because misses from previous rollouts offer limited predictive signal (\S\ref{sec:within_response_misses}), we set $B_{g,0}=0$ at the start of each rollout and retain only $A$ across rollouts. We evaluate $r\in\{4,8,16,32,64\}$ in Appendix~\ref{app:rank_topk}.

\subsection{Updating the Correction Every Round}
\label{submethod3}

Target verification provides $p_c$, which is ordinarily used only to accept or reject draft proposals. SPARK additionally uses it to align $q_c$ with $p_c$. Since a round terminates when the proposals fail to cover the target's token, we minimize the forward KL divergence $\mathcal{L}_c=D_{\mathrm{KL}}(p_c\,\|\,q_c)$, which, unlike the reverse KL, strongly penalizes target mass left uncovered by the draft and bounds the miss probability (Proposition~\ref{prop:coverage}). Its gradient with respect to the corrected logits is (Lemma~\ref{lem:logit_grad})
\begin{equation}
    \frac{\partial\mathcal{L}_c}{\partial\widetilde{z}_c(v)}
    =q_c(v)-p_c(v).
    \label{eq:spark_feedback}
\end{equation}

Let $C_{g,t}$ denote the draft positions selected to update $B_{g,t}$ after round $t$ for prompt $g$. We include positions along the verified path and rejected proposals directly adjacent to it. For tree-based drafting, these include rejected branches neighboring the verified path; for block-based drafting, only the first rejected token is retained. Positions beyond a rejection or without a corresponding draft prediction are excluded. To obtain sparse feedback, we retain the coordinates in
\begin{equation}
    S_c=\operatorname{TopK}(p_c,k)\cup\operatorname{TopK}(q_c,k),
    \label{eq:spark_topk}
\end{equation}
which covers tokens the draft may overestimate or underestimate, yielding
\begin{equation}
    \delta_c=\bigl(q_c-p_c\bigr)\odot\mathbf{1}_{S_c}.
    \label{eq:spark_sparse_feedback}
\end{equation}
This sparse feedback still gives a descent direction, with an error bounded by the probability mass outside $S_c$ (Proposition~\ref{prop:sparse}). We assess $k$ in Appendix~\ref{app:rank_topk}.

Since the corrected logits are linear in $B_{g,t}$, the per-context gradient is $\delta_c u_c^{\top}$ (Lemma~\ref{lem:param_grad}). SPARK applies one gradient step to the weighted average loss over $C_{g,t}$:
\begin{equation}
\begin{gathered}
B_{g,t+1} = B_{g,t}
-\eta\sum_{c\in C_{g,t}}
\frac{w_c}{W_{g,t}}\,\delta_c u_c^{\top},\\
W_{g,t}=\sum_{c\in C_{g,t}}w_c,
\end{gathered}
\label{eq:spark_update}
\end{equation}
where $w_c\geq0$ is the feedback weight and $\eta$ is the step size.\footnote{If $W_{g,t}=0$, $B_{g,t}$ remains unchanged.} This explicit update requires no backward pass or additional target forward pass, and the resulting $B_{g,t+1}$ is used from the next round onward. Unrolled from $B_{g,0}=0$, the correction at a new context is a sum of earlier feedback weighted by the similarity of projected features, and it changes only the logits of tokens that verification has already exposed for that prompt (Proposition~\ref{prop:kernel}).

\subsection{Learning the Projection Across Rollouts}
\label{submethod4}

Whereas $B_{g,t}$ captures rollout-specific errors, $A$ learns a feature subspace for logit correction that is shared across prompts and rollouts. We initialize $A$ from the target model's output projection $W\in\mathbb{R}^{V\times d}$: we take $r$ evenly spaced rows of $W$ and orthonormalize them by QR decomposition, so that the columns of $A$ form an orthonormal basis of a subspace spanned by token embeddings. $A$ is then updated throughout training. Treating $B_{g(c),t}$ as fixed, the gradient contribution of context $c$ is (Lemma~\ref{lem:param_grad})
\begin{equation}
    G^A_c=h_c\bigl(B_{g(c),t}^{\top}\delta_c\bigr)^{\top}.
    \label{eq:spark_projector_gradient}
\end{equation}
Because $B_{g,0}=0$, $A$ receives gradient only after within-rollout adaptation has begun, so it learns from the corrections actually required during generation; to first order, each step increases the similarity between contexts whose feedback agrees (Remark~\ref{rem:A_alignment}). These gradients are weighted by $w_c$, accumulated during generation, and applied at the existing draft optimizer step. Both $A$ and the draft backbone remain fixed during generation, and GRPO is unchanged. SPARK therefore operates on two timescales: $B_{g,t}$ adapts within a rollout, while $A$ is learned with the draft model across rollouts.

SPARK modifies only the proposal distribution. Exact speculative decoding preserves the target distribution for any proposal distribution fixed before verification, provided that the acceptance rule uses the corresponding $q_c$. Since $B_{g,t}$ changes only between verification rounds, the proposals of each round are fixed before they are verified, and the guarantee holds (Theorem~\ref{thm:exact}); hence the distribution of the GRPO update is also unchanged (Corollary~\ref{cor:grpo}). Individual trajectories may differ from the baseline, but they follow the same target distribution. The correction costs $O(dr+Vr)$ per context, compared with $O(Vd)$ for the draft output head, and requires one $V\times r$ matrix per prompt and one shared $d\times r$ projection (Remark~\ref{rem:cost}).
\section{Experiments and Results}
\label{experiments}

We evaluate SPARK by addressing three research questions:

RQ1: Does SPARK improve generation efficiency and end-to-end GRPO training speed across different speculative decoding methods?

RQ2: Does SPARK reduce the recurring draft misses identified in Section~\ref{sec:analysis}, and how does its benefit change during a rollout?

RQ3: How do feedback timing, sharing scope, feedback selection, and the learned projection affect performance?

\subsection{Experimental Setup}
\label{sec:exp_setup}

\paragraph{Models and datasets.}
We evaluate seven models: Qwen2.5-1.5B-Instruct, Qwen2.5-3B-Instruct,
Qwen2.5-7B-Instruct, Qwen2.5-14B-Instruct, Qwen3-1.7B, Qwen3-4B,
and Llama3.1-8B-Instruct. Each model is trained on
GSM8K~\citep{cobbe2021training},
SimpleLR~\citep{zeng2025simplerlzoo}, and
DAPO-Math-17K~\citep{yu2025dapo},
yielding 21 model--dataset settings.
Complete training configurations are provided in
Appendix~\ref{sec:implementation_detail}.

\paragraph{Draft pretraining.}
We pretrain each draft model for one epoch on ShareGPT\footnote{\url{https://huggingface.co/datasets/Aeala/ShareGPT_Vicuna_unfiltered}} before GRPO training. Within each baseline--SPARK pair, both runs start from the same target and pretrained draft checkpoints.

\begin{figure}[h]
    \centering
    \includegraphics[width=\columnwidth]{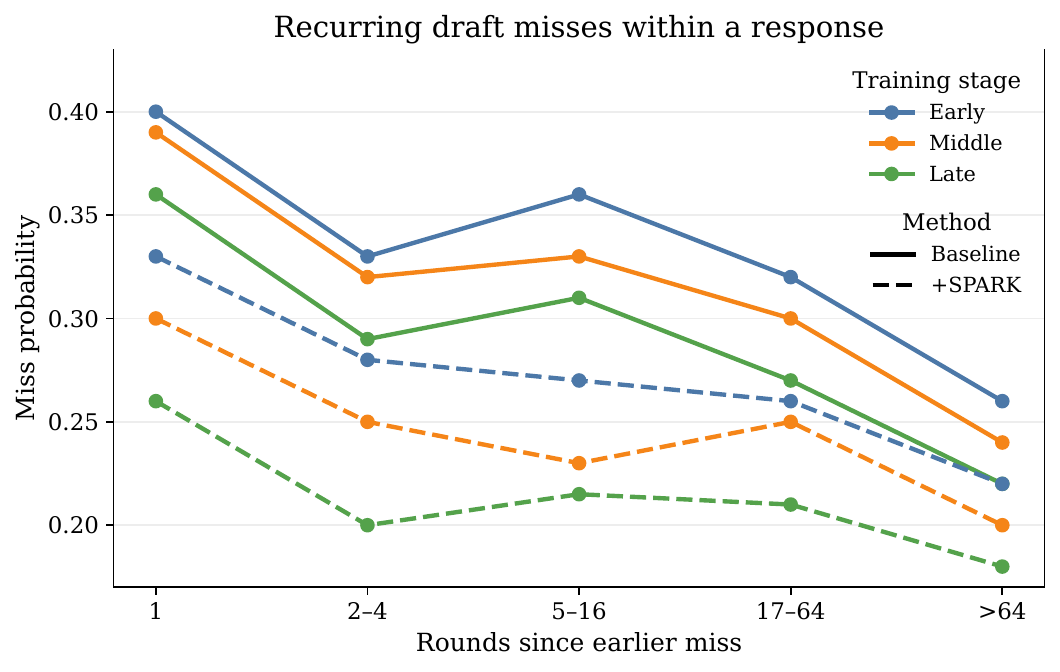}
    \caption{\textbf{Recurring draft misses within a response.}
    Miss probability across intervals of verification rounds after an earlier miss. Solid lines indicate the baseline drafter, dashed lines indicate SPARK, and colors denote early, middle, and late training stages.}
    \label{fig:recur_miss}
\end{figure}

\paragraph{Drafting methods.}
We use two drafting setups, referred to as \emph{EAGLE3-style} and \emph{DFlash-style}. EAGLE3-style uses the single-layer EAGLE3-style drafter and online draft training procedure from FastGRPO~\citep{zhang2025fastgrpo}. DFlash-style uses the DFlash draft architecture~\citep{chen2026dflash} and the online draft co-training implementation of NVIDIA~\citep{wang2026onlinedraft}. For each setup, we evaluate the baseline drafter and its SPARK variant under the same model initialization and GRPO configuration. SPARK preserves the original draft training and target verification procedures.

\paragraph{Metrics.}
We report average acceptance length (AAL), generation speedup (Gen.), and end-to-end GRPO training speedup (E2E). Generation speedup compares rollout generation time against autoregressive GRPO for the same model and dataset, normalized by the amount of generated output. E2E is the ratio of autoregressive GRPO training time to the method's training time at the same number of target policy updates, including rollout generation, online draft training, and target optimization. Draft pretraining is excluded from these measurements and reported separately. Both speedup measures include SPARK's runtime overhead. Since the two drafting setups construct proposals differently, we interpret AAL primarily within each setup.

\begin{figure}[h]
    \centering
    \includegraphics[width=\columnwidth]{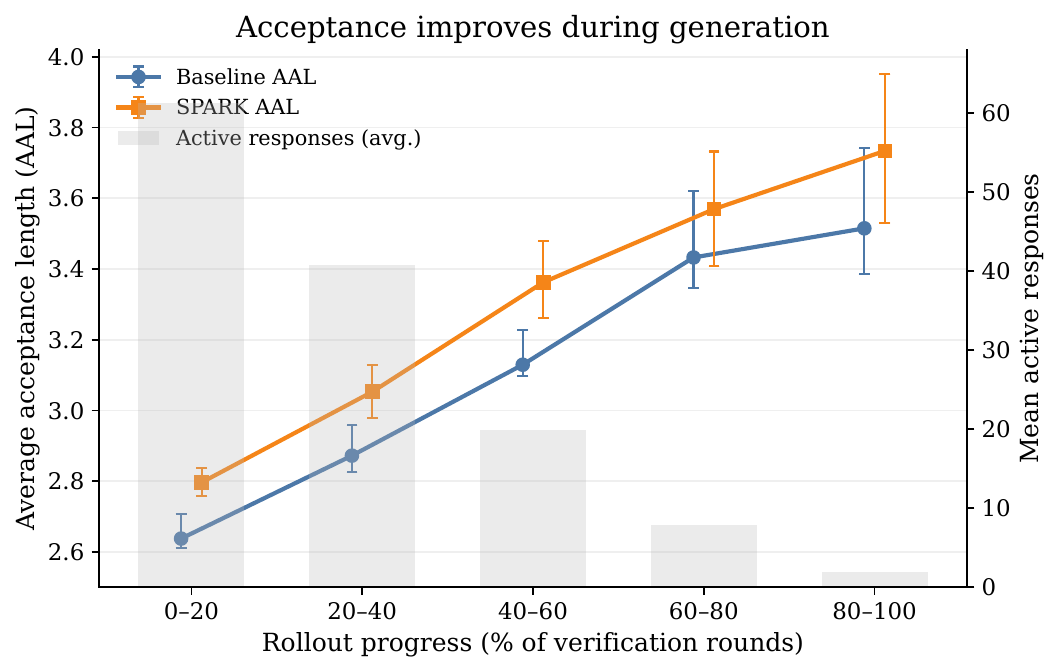}
    \caption{\textbf{Average acceptance length throughout a rollout.}
    Average acceptance length (AAL) with SPARK disabled and enabled over successive intervals of verification-round progress. Gray bars (right axis) show the average number of active responses across the two settings.}
    \label{fig:aal_progress}
\end{figure}

\subsection{RQ1: Overall Effectiveness}
\label{rq1}

\begin{table*}[t]
\centering
\footnotesize
\setlength{\tabcolsep}{2.5pt}
\renewcommand{\arraystretch}{1.08}
\definecolor{sparkbg}{RGB}{238,245,255}
\definecolor{headergray}{gray}{0.9}
\definecolor{gaincolor}{RGB}{145,0,170}

\providecommand{\gain}[1]{}
\renewcommand{\gain}[1]{%
  \,\textcolor{gaincolor}{%
    \raisebox{0.12ex}{\scriptsize\bfseries\boldmath$\uparrow$\,#1\%}%
  }%
}

\caption{\textbf{Overall GRPO efficiency across seven models and three datasets.}
For each draft method, we report generation speedup (Gen.) and end-to-end
speedup (E2E), both relative to autoregressive GRPO, and average acceptance
length (AAL), without (Orig.) and with SPARK (shaded).
\textbf{Bold} marks the better result in each pair; ties at the reported
precision are both bolded.
The last row reports the mean over all 21 model--dataset settings, with
$\uparrow$ indicating the relative improvement of the mean over Orig.}
\label{tab:main_results}

\resizebox{\textwidth}{!}{%
\begin{tabular}{
    @{}l
    *{2}{c}
    @{\hspace{5pt}}
    *{2}{c}
    @{\hspace{5pt}}
    *{2}{c}
    @{\hspace{10pt}}
    *{2}{c}
    @{\hspace{5pt}}
    *{2}{c}
    @{\hspace{5pt}}
    *{2}{c}
    @{}
}
\toprule
 & \multicolumn{6}{c}{EAGLE3-style}
 & \multicolumn{6}{c}{DFlash} \\
\cmidrule(lr){2-7}
\cmidrule(l){8-13}

\normalsize\textbf{Dataset}
 & \multicolumn{2}{c}{Gen.}
 & \multicolumn{2}{c}{E2E}
 & \multicolumn{2}{c}{AAL}
 & \multicolumn{2}{c}{Gen.}
 & \multicolumn{2}{c}{E2E}
 & \multicolumn{2}{c}{AAL} \\
\cmidrule(lr){2-3}
\cmidrule(lr){4-5}
\cmidrule(lr){6-7}
\cmidrule(lr){8-9}
\cmidrule(lr){10-11}
\cmidrule(l){12-13}

 & Orig. & \cellcolor{sparkbg}+SPARK
 & Orig. & \cellcolor{sparkbg}+SPARK
 & Orig. & \cellcolor{sparkbg}+SPARK
 & Orig. & \cellcolor{sparkbg}+SPARK
 & Orig. & \cellcolor{sparkbg}+SPARK
 & Orig. & \cellcolor{sparkbg}+SPARK \\
\midrule

\rowcolor{headergray}
\multicolumn{13}{@{}l@{}}{\textit{Qwen2.5-1.5B-Instruct}} \\
\hspace{6pt}GSM8K
 & 1.75 & \cellcolor{sparkbg}\textbf{2.26}
 & 1.69 & \cellcolor{sparkbg}\textbf{2.14}
 & 2.46 & \cellcolor{sparkbg}\textbf{3.24}
 & 2.23 & \cellcolor{sparkbg}\textbf{2.52}
 & 2.08 & \cellcolor{sparkbg}\textbf{2.27}
 & 2.58 & \cellcolor{sparkbg}\textbf{2.98} \\
\hspace{6pt}SimpleLR
 & 1.66 & \cellcolor{sparkbg}\textbf{2.14}
 & 1.61 & \cellcolor{sparkbg}\textbf{2.03}
 & 2.37 & \cellcolor{sparkbg}\textbf{3.12}
 & 2.13 & \cellcolor{sparkbg}\textbf{2.26}
 & 1.98 & \cellcolor{sparkbg}\textbf{2.06}
 & 2.58 & \cellcolor{sparkbg}\textbf{2.81} \\
\hspace{6pt}DAPO-Math-17K
 & 1.46 & \cellcolor{sparkbg}\textbf{1.87}
 & 1.44 & \cellcolor{sparkbg}\textbf{1.81}
 & 2.54 & \cellcolor{sparkbg}\textbf{3.33}
 & 1.78 & \cellcolor{sparkbg}\textbf{2.01}
 & 1.71 & \cellcolor{sparkbg}\textbf{1.90}
 & 2.63 & \cellcolor{sparkbg}\textbf{3.05} \\

\addlinespace[1.5pt]
\rowcolor{headergray}
\multicolumn{13}{@{}l@{}}{\textit{Qwen2.5-3B-Instruct}} \\
\hspace{6pt}GSM8K
 & 1.95 & \cellcolor{sparkbg}\textbf{2.28}
 & 1.83 & \cellcolor{sparkbg}\textbf{2.02}
 & 2.65 & \cellcolor{sparkbg}\textbf{3.28}
 & 2.24 & \cellcolor{sparkbg}\textbf{2.42}
 & 2.03 & \cellcolor{sparkbg}\textbf{2.16}
 & 2.53 & \cellcolor{sparkbg}\textbf{2.81} \\
\hspace{6pt}SimpleLR
 & 1.60 & \cellcolor{sparkbg}\textbf{2.68}
 & 1.54 & \cellcolor{sparkbg}\textbf{2.53}
 & 2.66 & \cellcolor{sparkbg}\textbf{2.77}
 & 1.71 & \cellcolor{sparkbg}\textbf{1.84}
 & 1.66 & \cellcolor{sparkbg}\textbf{1.75}
 & 2.50 & \cellcolor{sparkbg}\textbf{2.78} \\
\hspace{6pt}DAPO-Math-17K
 & 1.48 & \cellcolor{sparkbg}\textbf{1.90}
 & 1.45 & \cellcolor{sparkbg}\textbf{1.84}
 & 2.60 & \cellcolor{sparkbg}\textbf{3.43}
 & 1.61 & \cellcolor{sparkbg}\textbf{1.82}
 & 1.58 & \cellcolor{sparkbg}\textbf{1.74}
 & 2.44 & \cellcolor{sparkbg}\textbf{2.85} \\

\addlinespace[1.5pt]
\rowcolor{headergray}
\multicolumn{13}{@{}l@{}}{\textit{Qwen2.5-7B-Instruct}} \\
\hspace{6pt}GSM8K
 & 2.16 & \cellcolor{sparkbg}\textbf{3.71}
 & 2.06 & \cellcolor{sparkbg}\textbf{3.60}
 & 2.99 & \cellcolor{sparkbg}\textbf{3.22}
 & 2.26 & \cellcolor{sparkbg}\textbf{2.64}
 & 2.16 & \cellcolor{sparkbg}\textbf{2.43}
 & 2.66 & \cellcolor{sparkbg}\textbf{3.21} \\
\hspace{6pt}SimpleLR
 & 1.54 & \cellcolor{sparkbg}\textbf{2.53}
 & 1.36 & \cellcolor{sparkbg}\textbf{2.52}
 & 3.09 & \cellcolor{sparkbg}\textbf{3.31}
 & 1.51 & \cellcolor{sparkbg}\textbf{1.72}
 & 1.49 & \cellcolor{sparkbg}\textbf{1.66}
 & 2.82 & \cellcolor{sparkbg}\textbf{3.29} \\
\hspace{6pt}DAPO-Math-17K
 & 1.88 & \cellcolor{sparkbg}\textbf{2.50}
 & 1.81 & \cellcolor{sparkbg}\textbf{2.35}
 & 2.31 & \cellcolor{sparkbg}\textbf{3.14}
 & 2.28 & \cellcolor{sparkbg}\textbf{2.62}
 & 2.15 & \cellcolor{sparkbg}\textbf{2.39}
 & 2.33 & \cellcolor{sparkbg}\textbf{2.74} \\

\addlinespace[1.5pt]
\rowcolor{headergray}
\multicolumn{13}{@{}l@{}}{\textit{Qwen2.5-14B-Instruct}} \\
\hspace{6pt}GSM8K
 & 1.98 & \cellcolor{sparkbg}\textbf{2.56}
 & 1.89 & \cellcolor{sparkbg}\textbf{2.39}
 & 2.65 & \cellcolor{sparkbg}\textbf{3.51}
 & 2.46 & \cellcolor{sparkbg}\textbf{2.80}
 & 2.28 & \cellcolor{sparkbg}\textbf{2.53}
 & 2.68 & \cellcolor{sparkbg}\textbf{3.13} \\
\hspace{6pt}SimpleLR
 & 3.48 & \cellcolor{sparkbg}\textbf{3.62}
 & 3.16 & \cellcolor{sparkbg}\textbf{3.33}
 & 3.05 & \cellcolor{sparkbg}\textbf{3.85}
 & 3.53 & \cellcolor{sparkbg}\textbf{3.93}
 & 3.14 & \cellcolor{sparkbg}\textbf{3.37}
 & 2.96 & \cellcolor{sparkbg}\textbf{3.39} \\
\hspace{6pt}DAPO-Math-17K
 & 2.14 & \cellcolor{sparkbg}\textbf{2.74}
 & 2.04 & \cellcolor{sparkbg}\textbf{2.56}
 & 2.19 & \cellcolor{sparkbg}\textbf{2.88}
 & 2.59 & \cellcolor{sparkbg}\textbf{2.93}
 & 2.40 & \cellcolor{sparkbg}\textbf{2.64}
 & 2.21 & \cellcolor{sparkbg}\textbf{2.58} \\

\addlinespace[1.5pt]
\rowcolor{headergray}
\multicolumn{13}{@{}l@{}}{\textit{Qwen3-1.7B}} \\
\hspace{6pt}GSM8K
 & 1.41 & \cellcolor{sparkbg}\textbf{1.92}
 & 1.39 & \cellcolor{sparkbg}\textbf{1.84}
 & 2.53 & \cellcolor{sparkbg}\textbf{3.54}
 & 1.60 & \cellcolor{sparkbg}\textbf{1.80}
 & 1.55 & \cellcolor{sparkbg}\textbf{1.70}
 & 2.42 & \cellcolor{sparkbg}\textbf{2.81} \\
\hspace{6pt}SimpleLR
 & 1.55 & \cellcolor{sparkbg}\textbf{1.99}
 & 1.50 & \cellcolor{sparkbg}\textbf{1.90}
 & 2.36 & \cellcolor{sparkbg}\textbf{3.12}
 & 1.54 & \cellcolor{sparkbg}\textbf{1.59}
 & 1.50 & \cellcolor{sparkbg}\textbf{1.53}
 & 2.27 & \cellcolor{sparkbg}\textbf{2.43} \\
\hspace{6pt}DAPO-Math-17K
 & 1.76 & \cellcolor{sparkbg}\textbf{2.23}
 & 1.69 & \cellcolor{sparkbg}\textbf{2.10}
 & 2.47 & \cellcolor{sparkbg}\textbf{3.20}
 & 1.94 & \cellcolor{sparkbg}\textbf{2.17}
 & 1.83 & \cellcolor{sparkbg}\textbf{2.00}
 & 2.38 & \cellcolor{sparkbg}\textbf{2.72} \\

\addlinespace[1.5pt]
\rowcolor{headergray}
\multicolumn{13}{@{}l@{}}{\textit{Qwen3-4B}} \\
\hspace{6pt}GSM8K
 & 1.78 & \cellcolor{sparkbg}\textbf{2.26}
 & 1.70 & \cellcolor{sparkbg}\textbf{2.11}
 & 2.62 & \cellcolor{sparkbg}\textbf{3.40}
 & 2.04 & \cellcolor{sparkbg}\textbf{2.31}
 & 1.91 & \cellcolor{sparkbg}\textbf{2.10}
 & 2.50 & \cellcolor{sparkbg}\textbf{2.90} \\
\hspace{6pt}SimpleLR
 & 1.63 & \cellcolor{sparkbg}\textbf{2.09}
 & 1.58 & \cellcolor{sparkbg}\textbf{1.97}
 & 2.59 & \cellcolor{sparkbg}\textbf{3.40}
 & 1.60 & \cellcolor{sparkbg}\textbf{1.61}
 & \textbf{1.55} & \cellcolor{sparkbg}\textbf{1.55}
 & 2.46 & \cellcolor{sparkbg}\textbf{2.57} \\
\hspace{6pt}DAPO-Math-17K
 & 1.86 & \cellcolor{sparkbg}\textbf{2.35}
 & 1.78 & \cellcolor{sparkbg}\textbf{2.21}
 & 2.41 & \cellcolor{sparkbg}\textbf{3.12}
 & 2.08 & \cellcolor{sparkbg}\textbf{2.34}
 & 1.95 & \cellcolor{sparkbg}\textbf{2.13}
 & 2.30 & \cellcolor{sparkbg}\textbf{2.66} \\

\addlinespace[1.5pt]
\rowcolor{headergray}
\multicolumn{13}{@{}l@{}}{\textit{Llama3.1-8B-Instruct}} \\
\hspace{6pt}GSM8K
 & 1.75 & \cellcolor{sparkbg}\textbf{2.22}
 & 1.68 & \cellcolor{sparkbg}\textbf{2.09}
 & 2.17 & \cellcolor{sparkbg}\textbf{2.83}
 & 2.01 & \cellcolor{sparkbg}\textbf{2.27}
 & 1.89 & \cellcolor{sparkbg}\textbf{2.06}
 & 2.19 & \cellcolor{sparkbg}\textbf{2.55} \\
\hspace{6pt}SimpleLR
 & 1.74 & \cellcolor{sparkbg}\textbf{2.21}
 & 1.66 & \cellcolor{sparkbg}\textbf{2.08}
 & 2.30 & \cellcolor{sparkbg}\textbf{3.01}
 & 1.99 & \cellcolor{sparkbg}\textbf{2.22}
 & 1.85 & \cellcolor{sparkbg}\textbf{2.02}
 & 2.35 & \cellcolor{sparkbg}\textbf{2.71} \\
\hspace{6pt}DAPO-Math-17K
 & 1.98 & \cellcolor{sparkbg}\textbf{2.50}
 & 1.86 & \cellcolor{sparkbg}\textbf{2.32}
 & 1.89 & \cellcolor{sparkbg}\textbf{2.46}
 & 2.21 & \cellcolor{sparkbg}\textbf{2.48}
 & 2.05 & \cellcolor{sparkbg}\textbf{2.23}
 & 1.90 & \cellcolor{sparkbg}\textbf{2.20} \\

\midrule
\textbf{Average}
 & 1.84
 & \cellcolor{sparkbg}\textbf{2.43}\gain{32.7}
 & 1.75
 & \cellcolor{sparkbg}\textbf{2.31}\gain{32.2}
 & 2.52
 & \cellcolor{sparkbg}\textbf{3.47}\gain{37.7}
 & 2.06
 & \cellcolor{sparkbg}\textbf{2.30}\gain{11.4}
 & 1.94
 & \cellcolor{sparkbg}\textbf{2.11}\gain{8.5}
 & 2.46
 & \cellcolor{sparkbg}\textbf{2.82}\gain{14.5} \\
\bottomrule
\end{tabular}%
}
\end{table*}

Table~\ref{tab:main_results} compares EAGLE3-style and
DFlash-style drafting, each with and without SPARK,
across model--dataset settings. We examine both
generation speedup and end-to-end GRPO training
speedup to evaluate the benefits of SPARK.

With EAGLE3-style drafting, SPARK achieves an
average generation speedup of $2.43\times$ and
an E2E speedup of $2.31\times$, improving upon
the corresponding baseline by $32.7\%$ and
$32.2\%$, respectively. With DFlash-style drafting,
SPARK achieves average speedups of $2.30\times$
for generation and $2.11\times$ for E2E training,
corresponding to improvements of $11.4\%$ and
$8.5\%$ over the baseline. These results demonstrate
that SPARK improves both rollout generation and
overall training efficiency across different
drafting architectures.

\subsection{RQ2: Understanding the Improvements}
\label{rq2}

\paragraph{Recurring draft misses.}

We investigate whether SPARK reduces the repeated draft misses observed in Section~\ref{sec:analysis}. Using the same frozen checkpoints, we compare how likely the baseline drafter and SPARK are to miss a token again after an earlier miss in the same response. Figure~\ref{fig:recur_miss} compares these probabilities across training stages and distances between verification rounds. SPARK shows lower miss probabilities than the baseline throughout the measured intervals, indicating that verification feedback can help reduce repeated errors during generation.

\paragraph{Improvements throughout generation.}

We further investigate how SPARK improves draft acceptance as generation progresses. Using the same frozen target, drafter, and learned projection, we compare the baseline with SPARK to evaluate the benefit of incorporating verification feedback during each rollout. Figure~\ref{fig:aal_progress} shows that SPARK consistently achieves higher average acceptance length (AAL) throughout generation. Averaging the AAL values plotted in the figure, SPARK improves AAL from $3.12$ to $3.30$, corresponding to a gain of $5.95\%$. This improvement demonstrates the benefit of adapting draft predictions using feedback from previous verification rounds, allowing SPARK to produce more accurate drafts without updating the underlying model weights. The gray bars illustrate how the number of active responses gradually decreases as generation progresses. Notably, SPARK maintains its advantage over the baseline even as fewer responses remain active. These results highlight the effectiveness of prompt-specific adaptation in improving draft acceptance throughout GRPO rollouts.

\subsection{RQ3: Ablation Studies}
\label{rq3}

We examine the key design choices of SPARK using Qwen2.5-3B-Instruct on GSM8K with EAGLE3-style drafting. Starting from the full method, we change one component at a time while keeping the remaining experimental settings fixed. Full SPARK shares a correction matrix $B_{g,t}$ among responses to the same prompt, updates it after every verification round, and resets it at the start of each rollout. It also uses feedback from both accepted and eligible rejected positions and learns the projection $A$ across rollouts. As shown in Table~\ref{tab:spark_ablations}, full SPARK achieves the highest AAL (3.28) and end-to-end (E2E) training speedup (2.02$\times$) among the evaluated configurations.

\paragraph{Feedback update frequency.}
Updating the correction after every verification round provides the best results. Reducing the update frequency to once every 5, 10, or 15 rounds decreases AAL from 3.28 to 3.11, 3.01, and 2.85, respectively. The corresponding E2E speedups also decline from 2.02$\times$ to 1.95$\times$, 1.92$\times$, and 1.84$\times$. These results show that promptly incorporating verification feedback improves draft acceptance while also benefiting overall training efficiency.

\paragraph{Correction sharing and reset.}
We next examine how the correction is shared across responses. Assigning a separate correction to each response achieves an AAL of 3.24 and an E2E speedup of 1.96$\times$, whereas sharing one correction across the entire rollout batch reduces them to 3.12 and 1.92$\times$. The prompt-specific design performs best, allowing responses to the same prompt to benefit from shared feedback without mixing corrections across unrelated prompts. Resetting the correction between rollouts is also beneficial: carrying it across rollouts reduces AAL to 3.05 and E2E speedup to 1.85$\times$. This supports adapting the correction to the current rollout rather than retaining it across different rollouts.

\paragraph{Feedback selection and learned projection.}
Using accepted positions alone reduces AAL from $3.28$ to $3.02$ and E2E speedup from $2.02\times$ to $1.94\times$, highlighting the benefit of incorporating eligible rejected proposals. We also examine how the number of selected feedback tokens $k$ affects performance. Increasing $k$ from $8$ to $16$ improves both AAL and E2E speedup, while larger values provide little additional benefit. Although $k=32$ achieves a slightly higher AAL ($3.29$), $k=16$ delivers the highest E2E speedup ($2.02\times$), offering a better balance between draft acceptance and training efficiency. Detailed results are provided in Appendix~\ref{app:rank_topk}. Finally, freezing the projection $A$ reduces AAL to $2.88$ and E2E speedup to $1.91\times$, demonstrating the importance of learning the projection across rollouts.

\begin{table}[t]
\centering
\caption{\textbf{SPARK ablation results} on Qwen2.5-3B-Instruct with GSM8K using EAGLE3-style drafting. Each ablation changes one component relative to full SPARK. E2E denotes end-to-end training speedup over autoregressive GRPO.}
\label{tab:spark_ablations}
\setlength{\tabcolsep}{5pt}
\footnotesize
\begin{tabular}{@{}lcc@{}}
\toprule
Configuration & AAL & E2E ($\times$) \\
\midrule
\textbf{Full SPARK} & \textbf{3.28} & \textbf{2.02} \\
\midrule
Update every 5 rounds & 3.11 & 1.95 \\
Update every 10 rounds & 3.01 & 1.92 \\
Update every 15 rounds & 2.85 & 1.84 \\
\midrule
Per-response correction & 3.24 & 1.96 \\
Batch-wide correction & 3.12 & 1.92 \\
\midrule
Carry correction across rollouts & 3.05 & 1.85 \\
\midrule
Accepted feedback only & 3.02 & 1.94 \\
Frozen projection $A$ & 2.88 & 1.91 \\
\bottomrule
\end{tabular}
\end{table}

\section{Related Work}
\label{sec:related}

\paragraph{Efficient RL training and speculative decoding.}
Rollout generation dominates the cost of GRPO-style RL~\citep{shao2024deepseekmath,guo2025deepseekr1}. Asynchronous and partial rollouts raise throughput at the cost of on-policyness~\citep{fu2025areal,kimi2025k15,sheng2025hybridflow}, whereas speculative decoding preserves the target distribution~\citep{leviathan2023fast,chen2023accelerating}. Drafters include auxiliary heads~\citep{cai2024medusa,li2024eagle,li2025eagle3}, token trees~\citep{miao2024specinfer}, block drafters~\citep{chen2026dflash}, and retrieval~\citep{he2024rest,oliaro2024suffixdecoding}. Learned drafters stay fixed during generation, and verification only accepts or rejects their proposals.

\paragraph{Speculative decoding for RL.}
RL drafters must track a changing policy. Non-parametric methods reuse earlier rollouts as drafts~\citep{liu2026specrl,shao2025das}; DAS updates its suffix tree within a rollout but matches only verbatim repeats. Learned drafters are trained between rollouts or RL steps~\citep{zhang2025fastgrpo,hu2026tlt,chen2026respec,iso2026rlacceleration,wang2026onlinedraft}, including FastGRPO, whose concurrency-aware budget we adopt. None of them can correct errors that recur within the current rollout (Section~\ref{sec:analysis}).

\paragraph{Online and test-time adaptation.}
Online speculative decoding adapts drafter weights across queries~\citep{liu2024online}, and DistillSpec aligns the drafter by distillation~\citep{zhou2024distillspec}. Dynamic evaluation, test-time training and semi-amortized inference adapt a model to each instance~\citep{krause2018dynamic,sun2020test,kim2018semi}. SPARK brings this per-instance adaptation to speculative RL: to our knowledge, it is the first learned, distributional adaptation of the drafter within a rollout. After every verification round, it makes a closed-form low-rank update from target distributions the verifier has already computed, including at rejected proposals, and shares it within each GRPO group. It extends existing drafters without changing their training or the target distribution.
\section{Conclusion}
We introduced SPARK, a lightweight mechanism that uses verification feedback to adapt draft predictions during GRPO rollouts. Prompt-shared, low-rank corrections reduce recurring draft misses while preserving the target distribution and requiring no additional target forward passes. Across 21 model--dataset settings, SPARK improves end-to-end training speedup by 32.2\% over EAGLE3-style drafting and 8.5\% over DFlash-style drafting. These results show that feedback already available during verification can improve rollout efficiency beyond online draft training alone.

\section*{Limitations}

Our evaluation focuses on mathematical reasoning datasets (GSM8K, SimpleLR and DAPO-Math-17K). Other RL domains, such as code generation, agentic tasks or open-ended dialogue, may show different patterns of draft misses and remain to be evaluated. We apply SPARK to two drafting approaches, EAGLE3-style trees and DFlash blocks; other drafters, such as Medusa-style heads, separate small models or retrieval-based methods, could also be combined with SPARK but are not tested here. Our experiments use Qwen and Llama models of up to 14B parameters with GRPO. Larger models, other model families and other group-based RL algorithms are left to future work.

\bibliography{custom}
\appendix
\section{Implementation Detail}
\label{sec:implementation_detail}

\paragraph{Training pipeline.}
We format each example with the model's native chat template and request
step-by-step reasoning followed by a boxed answer.  For each prompt $g$, we
form a GRPO group of $G=8$ responses and use
$R=R_{\mathrm{acc}}+0.2R_{\mathrm{fmt}}$, where answer correctness is determined
by symbolic equivalence and $R_{\mathrm{fmt}}$ rewards the expected response
format.  Rewards are standardized independently within each group.  Within
each drafting setup, the baseline and SPARK use the same target update,
sampling configuration, and initialization.  Table~\ref{tab:implementation_hparams}
summarizes the shared settings.

\begin{table*}[t]
\centering
\caption{\textbf{Implementation settings.} Settings are shared by a baseline
and its SPARK counterpart unless stated otherwise.}
\label{tab:implementation_hparams}
\small
\setlength{\tabcolsep}{3pt}
\begin{tabular}{@{}p{0.48\textwidth}p{0.45\textwidth}@{}}
\toprule
Setting & Value \\
\midrule
GRPO group size $G$ & $8$ responses per prompt \\
Target / online draft learning rate & $1\times10^{-6}$ / $1\times10^{-4}$ \\
GRPO $\epsilon$ / $\beta$ & $0.1$ / $0.04$ \\
Sampling & temperature $1.0$, top-$p$ $0.95$ \\
Maximum prompt / response length & $2{,}048$ / $2{,}048$ tokens \\
Draft pretraining & ShareGPT, $1$ epoch \\
Pretraining learning rate / length & $5\times10^{-5}$ / $2{,}048$ tokens \\
SPARK rank $r$ / feedback top-$k$ & $8$ / $16$ \\
Adaptation / projector learning rate & $0.01$ / $1\times10^{-4}$ \\
Feedback weights (path / rejected) & $1.0$ / $1.0$ \\
Update / scope / reset & every round / prompt group / rollout \\
\bottomrule
\end{tabular}
\end{table*}

\paragraph{Draft architectures and training.}
We evaluate SPARK with both tree- and block-based drafting.  The EAGLE3-style
drafter uses one decoder layer, shares the target embedding and language-model
head, and proposes a token tree over the full target vocabulary.  Its draft
objective is
$2\mathcal{L}_{\mathrm{SmoothL1}}+0.1\mathcal{L}_{\mathrm{softCE}}$,
combining next-feature prediction and target-distribution distillation.  The
DFlash-style drafter uses three layers and a block size of 16 (one anchor and
15 predicted positions).  It consumes target features, shares the target
embedding and language-model head, and is trained with position-weighted soft
cross-entropy.  At inference, it forms greedy block proposals that are checked
by target-model rejection sampling.  Both drafters are pretrained on ShareGPT
for one epoch and continue to be updated online.  For each architecture, the
baseline and SPARK start from the same target and pretrained draft checkpoints
and use the same draft-training objective.

\paragraph{SPARK implementation.}
For each prompt group $g$, SPARK maintains a rollout-local correction
$B_{g,t}$ shared only by the $G$ responses generated from that prompt.  After
verification round $t$, we collect the eligible contexts in $C_{g,t}$.  For
tree drafting, these comprise states on the verified path and eligible rejected
children; for block drafting, they comprise the accepted prefix and the first
rejected token.  Each context approximates the forward-KL gradient on the union
of the target and draft top-16 tokens, together with a tail bucket.
Contributions are aggregated within the prompt group and normalized by its
total feedback weight $W_{g,t}$ before updating $B_{g,t}$.  Feedback is
therefore shared among responses to the same prompt, but never across prompts,
and $B_{g,0}$ is reset to zero at the start of each rollout.

The projection $A$ is shared across prompt groups and retained across rollouts.
We initialize it from evenly spaced rows of the target output projection and
orthonormalize it with QR decomposition.  Both $A$ and $B_{g,t}$ are maintained
in FP32.  Gradients for $A$ are accumulated during generation and applied only
at the existing draft optimizer step; neither $A$ nor the draft backbone is
updated during rollout generation.  Target probabilities are reused from
verification, so SPARK requires no additional target forward or backward pass.

\paragraph{Timing protocol.}
A baseline and its SPARK counterpart use identical initialization, training,
and decoding settings.  Generation throughput includes proposal construction,
target verification, and SPARK updates.  End-to-end time additionally includes
online draft training and target-policy optimization; thus, all SPARK runtime
overheads are included in the reported results.



\section{Theoretical Properties of SPARK}
\label{app:theory}

This appendix states and proves the properties of SPARK used in Section~\ref{method}, with the notation of \S\ref{sec:spark_setup}. We write $q_c=\operatorname{softmax}(\widetilde z_c)$ for the corrected draft distribution, $\bar\delta_c=q_c-p_c$ for the full logit gradient, and $\delta_c=\bar\delta_c\odot\mathbf 1_{S_c}$ for the sparse feedback. For a token set $S$, $S^{\mathsf c}$ is its complement and $p(S)=\sum_{v\in S}p(v)$. Within a rollout, $A$ is fixed, so $u_c=A^{\top}h_c$ does not change once $c$ has been drafted.

\subsection{The Feedback Objective}
\label{app:theory_objective}

\begin{lemma}[Logit gradient]
\label{lem:logit_grad}
For a fixed distribution $p$ and $q=\operatorname{softmax}(\widetilde z)$,
\[
\nabla_{\widetilde z}\,D_{\mathrm{KL}}(p\,\|\,q)=q-p .
\]
\end{lemma}

\begin{proof}
Since $\log q(v)=\widetilde z(v)-\log\sum_u e^{\widetilde z(u)}$,
\[
\begin{aligned}
D_{\mathrm{KL}}(p\|q)={}&\textstyle\sum_v p(v)\log p(v)-\langle p,\widetilde z\rangle\\
&+\textstyle\log\sum_u e^{\widetilde z(u)} .
\end{aligned}
\]
The first term is constant in $\widetilde z$, the second has gradient $-p$, and the third has gradient $q$.
\end{proof}

A \emph{miss} occurs when the token selected by the target lies outside the set $K$ of proposed tokens (Section~\ref{sec:analysis}), which happens with probability $1-p(K)$. The next result explains why SPARK uses the forward KL.

\begin{proposition}[Coverage]
\label{prop:coverage}
Let $p,q$ be distributions over the vocabulary.

\noindent(i) For every token $v$,
\[
D_{\mathrm{KL}}(p\|q)\ \ge\ \mathrm{kl}\bigl(p(v)\,\|\,q(v)\bigr),
\]
where $\mathrm{kl}(a\|b)=a\log\frac{a}{b}+(1-a)\log\frac{1-a}{1-b}$. Hence $D_{\mathrm{KL}}(p\|q)\to\infty$ as $q(v)\to0$ whenever $p(v)>0$, whereas $D_{\mathrm{KL}}(q\|p)$ can stay finite when $q(v)=0$.

\noindent(ii) For every proposal set $K$,
\[
1-p(K)\ \le\ 1-q(K)+\sqrt{\tfrac12 D_{\mathrm{KL}}(p\|q)} .
\]
\end{proposition}

\begin{proof}
(i) Merging the vocabulary into the two outcomes $\{v\}$ and $\{v\}^{\mathsf c}$ cannot increase the KL divergence (data-processing inequality), which gives the bound. Its first term diverges as $q(v)\to0$ when $p(v)>0$. Conversely, if $q(v)=0$ and $p>0$ on the support of $q$, then $D_{\mathrm{KL}}(q\|p)=\sum_{u\neq v}q(u)\log\frac{q(u)}{p(u)}<\infty$.

(ii) By Pinsker's inequality,
\[
|p(K)-q(K)|\le\mathrm{TV}(p,q)\le\sqrt{\tfrac12 D_{\mathrm{KL}}(p\|q)} .\qedhere
\]
\end{proof}

Part (i) shows that the forward KL strongly penalizes a draft that assigns little probability to a token the target considers likely, which is the failure that causes a miss; the reverse KL does not. Part (ii) shows that when the proposals carry most of the draft mass ($q(K)\approx1$), reducing $D_{\mathrm{KL}}(p\|q)$ tightens an upper bound on the miss probability.

\subsection{Gradients and Sparse Feedback}
\label{app:theory_sparse}

\begin{lemma}[Parameter gradients]
\label{lem:param_grad}
For a context $c$ of prompt $g$, with $\widetilde z_c=z_c+B_{g,t}A^{\top}h_c$,
\[
\nabla_{B_{g,t}}\mathcal L_c=\bar\delta_c\,u_c^{\top},
\]
and, treating $B_{g,t}$ as fixed,
\[
\nabla_{A}\mathcal L_c=h_c\bigl(B_{g,t}^{\top}\bar\delta_c\bigr)^{\top}.
\]
Eqs.~\eqref{eq:spark_update} and~\eqref{eq:spark_projector_gradient} use these gradients with $\bar\delta_c$ replaced by $\delta_c$.
\end{lemma}

\begin{proof}
By Lemma~\ref{lem:logit_grad}, $d\mathcal L_c=\langle\bar\delta_c,d\widetilde z_c\rangle$. For $B$, $d\widetilde z_c=(dB)\,u_c$, so $d\mathcal L_c=\langle\bar\delta_cu_c^{\top},dB\rangle$. For $A$ with $B$ fixed, $d\widetilde z_c=B(dA)^{\top}h_c$, so
\[
d\mathcal L_c=\langle B^{\top}\bar\delta_c,(dA)^{\top}h_c\rangle
=\langle h_c(B^{\top}\bar\delta_c)^{\top},dA\rangle .\qedhere
\]
\end{proof}

The gradient for $A$ ignores the dependence of $B_{g,t}$ on $A$ through earlier features $u_{c'}$. SPARK uses this truncated gradient, which is why Eq.~\eqref{eq:spark_projector_gradient} treats $B_{g,t}$ as fixed.

\begin{proposition}[Sparse feedback]
\label{prop:sparse}
(i) For each context $c$, $-\delta_cu_c^{\top}$ is a descent direction for $\mathcal L_c$:
\[
\bigl\langle\nabla_B\mathcal L_c,\ \delta_cu_c^{\top}\bigr\rangle
=\|u_c\|^2\!\sum_{v\in S_c}\!\bar\delta_c(v)^2\ \ge\ 0 .
\]
(ii) Its error is bounded by the mass outside $S_c$:
\[
\bigl\|(\bar\delta_c-\delta_c)\,u_c^{\top}\bigr\|_F
\le\|u_c\|\,\epsilon_c,
\]
where $\epsilon_c=p_c(S_c^{\mathsf c})+q_c(S_c^{\mathsf c})$, and
\[
\begin{aligned}
p_c(S_c^{\mathsf c})&\le1-p_c\bigl(\operatorname{TopK}(p_c,k)\bigr),\\
q_c(S_c^{\mathsf c})&\le1-q_c\bigl(\operatorname{TopK}(q_c,k)\bigr).
\end{aligned}
\]
(iii) Let $\mathcal L_{g,t}=\sum_{c\in C_{g,t}}\frac{w_c}{W_{g,t}}\mathcal L_c$, with gradient $\nabla$, and let $\widehat\nabla=\sum_{c}\frac{w_c}{W_{g,t}}\delta_cu_c^{\top}$ be the update direction of Eq.~\eqref{eq:spark_update}. If $\|\widehat\nabla-\nabla\|_F<\|\nabla\|_F$, then $-\widehat\nabla$ is a descent direction for $\mathcal L_{g,t}$.
\end{proposition}

\begin{proof}
(i) By Lemma~\ref{lem:param_grad}, the inner product equals $\|u_c\|^2\langle\bar\delta_c,\delta_c\rangle$, and $\langle\bar\delta_c,\delta_c\rangle=\sum_{v\in S_c}\bar\delta_c(v)^2$.

(ii) The left side equals $\|u_c\|\,\|\bar\delta_c\odot\mathbf 1_{S_c^{\mathsf c}}\|_2$, and
\[
\|\bar\delta_c\odot\mathbf 1_{S_c^{\mathsf c}}\|_2
\le\!\sum_{v\notin S_c}\!|q_c(v)-p_c(v)|
\le\epsilon_c .
\]
The tail bounds hold because $S_c$ contains both top-$k$ sets.

(iii) By Cauchy--Schwarz,
\[
\begin{aligned}
\langle\widehat\nabla,\nabla\rangle
&=\|\nabla\|_F^2+\langle\widehat\nabla-\nabla,\nabla\rangle\\
&\ge\|\nabla\|_F\bigl(\|\nabla\|_F-\|\widehat\nabla-\nabla\|_F\bigr)>0 .
\end{aligned}
\qedhere
\]
\end{proof}

When $p_c$ is truncated by top-$p$ sampling to at most $k$ tokens, $p_c(S_c^{\mathsf c})=0$, and the error comes only from the draft tail.

\subsection{What the Correction Computes}
\label{app:theory_kernel}

Unrolling Eq.~\eqref{eq:spark_update} from $B_{g,0}=0$ gives a closed form for the correction. Let
\[
\kappa(c,c')=\langle u_c,u_{c'}\rangle=h_c^{\top}AA^{\top}h_{c'},
\]
and let $\mathcal C_{g,t}=\bigcup_{s<t}C_{g,s}$ be the feedback contexts of prompt $g$ before round $t$, with $\lambda_c=\eta\,w_c/W_{g,s}$ for $c\in C_{g,s}$.

\begin{proposition}[Closed form]
\label{prop:kernel}
For any context $c'$ of prompt $g$ drafted in round $t$,
\[
B_{g,t}\,u_{c'}=-\sum_{c\in\mathcal C_{g,t}}\lambda_c\,\kappa(c,c')\,\delta_c .
\]
Consequently:

\noindent(i) \emph{Token locality.} The correction changes the logit of a token $v$ only if $v\in S_c$ for some $c\in\mathcal C_{g,t}$. The probability ratio of any two tokens outside this set is unchanged.

\noindent(ii) \emph{Prompt isolation.} $B_{g,t}$ depends only on feedback from responses to prompt $g$. Feedback from one response affects the other responses to the same prompt and no other prompt.

\noindent(iii) \emph{Context dependence.} Feedback from $c$ is applied at $c'$ with weight $\kappa(c,c')$. If $u_{c'}$ is orthogonal to $u_c$ for all $c\in\mathcal C_{g,t}$, the correction at $c'$ is zero.

\noindent(iv) \emph{Initialization.} With the orthonormal initialization of $A$ (\S\ref{submethod4}), $AA^{\top}$ is the orthogonal projector onto the span of $r$ token-embedding directions. Then $\kappa(c,c')$ is the inner product of the two draft features after projection onto that subspace.
\end{proposition}

\begin{proof}
By induction on $t$, using $B_{g,0}=0$ and Eq.~\eqref{eq:spark_update},
\[
B_{g,t}=-\sum_{c\in\mathcal C_{g,t}}\lambda_c\,\delta_cu_c^{\top}.
\]
Each $u_c$ is unchanged after $c$ is drafted because $A$ is fixed during the rollout. Multiplying by $u_{c'}$ gives the formula.
(i) Each $\delta_c$ is supported on $S_c$, so $B_{g,t}u_{c'}$ is supported on their union. If two tokens $v,v'$ receive no correction, then $q_{c'}(v)/q_{c'}(v')=e^{z_{c'}(v)-z_{c'}(v')}$, as without correction.
(ii) Only contexts of prompt $g$ appear in $\mathcal C_{g,t}$.
(iii) Immediate from the formula.
(iv) If $A$ has orthonormal columns, $AA^{\top}$ is the orthogonal projector onto $\operatorname{span}(A)$, which by construction is spanned by the selected rows of $W$.
\end{proof}

Proposition~\ref{prop:kernel} makes precise how one $B_{g,t}$ acts differently in different contexts (\S\ref{submethod1}). SPARK behaves like a kernel smoother over the feedback collected so far in the rollout, with the kernel defined by $A$. The correction is also local by construction: it moves only tokens that verification has already exposed for this prompt, and the reset $B_{g,0}=0$ discards this memory at the rollout boundary.

\begin{remark}[Effect of learning $A$]
\label{rem:A_alignment}
Substituting the closed form of $B_{g,t}$ into Eq.~\eqref{eq:spark_projector_gradient} gives
\[
G^A_{c'}=-\sum_{c\in\mathcal C_{g,t}}\lambda_c\,\langle\delta_c,\delta_{c'}\rangle\,h_{c'}u_c^{\top}.
\]
A plain gradient step $A\leftarrow A-\alpha G^A_{c'}$ changes the projected feature of $c'$ by
\[
\Delta u_{c'}=\alpha\|h_{c'}\|^2\!\sum_{c\in\mathcal C_{g,t}}\!\lambda_c\langle\delta_c,\delta_{c'}\rangle\,u_c .
\]
To first order, and ignoring the simultaneous change of the earlier $u_c$, this increases $\kappa(c,c')$ when the feedback at $c$ and $c'$ agrees ($\langle\delta_c,\delta_{c'}\rangle>0$) and decreases it when the feedback conflicts. Across rollouts, $A$ therefore learns which contexts should share corrections. This interprets the update direction; it is not a convergence result, and in practice $A$ is updated by the draft optimizer.
\end{remark}

\subsection{Exactness}
\label{app:theory_exact}

Let $\mathcal F_t$ denote all information available in the current rollout before round $t$. It contains the prompts and the committed tokens, proposals, verifier outputs and target distributions of rounds $0,\dots,t-1$. It also contains the parameters fixed before the rollout: the target policy, the draft model and $A$.

\begin{assumption}[Exact verification]
\label{ass:exact}
In each round, the verifier receives the proposals of every active response together with the distributions they were drawn from. Its randomness is independent of $\mathcal F_t$ and of the draft's randomness. Suppose the proposals of round $t$ are determined by $\mathcal F_t$ and independent draft randomness. Then, conditionally on $\mathcal F_t$ and on the tokens committed before it in that round, each token committed in round $t$ for response $i$ follows the target's sampling distribution given the prefix of response $i$. Tokens committed for different responses are conditionally independent given $\mathcal F_t$.
\end{assumption}

Assumption~\ref{ass:exact} is the standard guarantee of lossless speculative decoding. It is also what makes the baseline drafter exact, since that drafter conditions on previously committed tokens as well.

\begin{theorem}[Exactness]
\label{thm:exact}
Under Assumption~\ref{ass:exact}, the responses generated with SPARK have the same joint distribution as autoregressive sampling from the target. Given the prompts, the responses are independent, and each follows the target's sampling distribution. In particular, the distribution is the same as without SPARK.
\end{theorem}

\begin{proof}
\emph{Step 1.} $A$ is updated only at the draft optimizer step after the rollout, so it is fixed during the rollout and is part of $\mathcal F_0$.

\emph{Step 2.} $B_{g,t}$ is determined by $\mathcal F_t$. This holds for $t=0$, since $B_{g,0}=0$. By Eq.~\eqref{eq:spark_update}, $B_{g,t+1}$ is a function of $B_{g,t}$ and of $(w_c,\delta_c,u_c)$ for $c\in C_{g,t}$. These quantities are computed from round $t$, so they are determined by $\mathcal F_{t+1}$. The claim follows by induction.

\emph{Step 3.} In round $t$, $h_c$ and $z_c$ are computed from the committed prefixes and the current draft tree. They are therefore determined by $\mathcal F_t$ and the draft's randomness. By Steps 1 and 2, the same holds for $u_c$, $\widetilde z_c$ and $q_c$, and hence for the proposals. SPARK passes $q_c$ to the verifier wherever the acceptance rule uses the proposal distribution.

\emph{Step 4.} By Assumption~\ref{ass:exact}, each committed token follows, conditionally on everything generated before it, the target's sampling distribution given its own prefix, independently across responses. This conditional law depends on the past only through the token's own prefix. By the chain rule of probability, the joint distribution of all responses is the product of the target's sequence distributions. The same argument applies without SPARK, so the two distributions coincide.
\end{proof}

\begin{corollary}[GRPO is unaffected]
\label{cor:grpo}
The rewards, group-normalized advantages and GRPO policy-gradient estimate are functions of the prompts, the sampled responses and the target's log-probabilities. By Theorem~\ref{thm:exact}, their joint distribution is the same with and without SPARK. SPARK can change the wall-clock time of a GRPO step, but not the distribution of the update it computes.
\end{corollary}

\begin{remark}[Scope]
\label{rem:exact_scope}
(i) The guarantee is distributional. Individual trajectories differ from the baseline, because different proposals consume randomness differently.
(ii) The proof requires only that the correction used in round $t$ be determined before round $t$ is verified. It therefore also covers delayed or asynchronous updates, in which round $t$ uses $B_{g,s}$ for some $s\le t$.
(iii) The guarantee would not hold if the correction used information produced while verifying the current round, such as the target's own sample, before that round is committed. SPARK does not do this.
\end{remark}

\subsection{Cost}
\label{app:theory_cost}

\begin{remark}[Computation and memory]
\label{rem:cost}
For each context, computing $u_c$ costs $O(dr)$ and applying $B_{g,t}u_c$ costs $O(Vr)$. The draft output head costs $O(Vd)$, so the relative overhead is $O(r/d)$. By Proposition~\ref{prop:kernel}(i), $B_{g,t}$ has nonzero rows only for tokens in the feedback sets $S_c$ of the current rollout. The correction can therefore also be evaluated in $O(dr+m_{g,t}r)$, where $m_{g,t}$ is the number of such tokens and $m_{g,t}\le\min(V,\,2k|\mathcal C_{g,t}|)$. Each update in Eq.~\eqref{eq:spark_update} costs $O(2kr)$ per feedback context and needs no backward pass. Memory is one $V\times r$ matrix per prompt and one $d\times r$ projection.
\end{remark}

\section{Sensitivity to Correction Rank and Feedback Size}
\label{app:rank_topk}

We examine the sensitivity of SPARK to the correction rank $r$ and the number of feedback tokens $k$ using Qwen2.5-3B-Instruct on GSM8K with EAGLE3-style drafting. We vary one parameter at a time, keeping $k=16$ for the rank sweep and $r=8$ for the top-$k$ sweep. The default configuration uses $r=8$ and $k=16$. Table~\ref{tab:spark_rank_topk} reports average acceptance length (AAL) and end-to-end (E2E) training speedup over autoregressive GRPO.

\begin{table}[t]
\centering
\caption{\textbf{Sensitivity to correction rank and feedback size.} The default settings are $r=8$ and $k=16$. Each sweep varies one parameter while keeping the other fixed.}
\label{tab:spark_rank_topk}
\setlength{\tabcolsep}{7pt}
\footnotesize
\begin{tabular}{@{}lccc@{}}
\toprule
Parameter & Value & AAL & E2E ($\times$) \\
\midrule
\multirow{5}{*}{Rank $r$}
 & $4$  & 3.20 & 1.99 \\
 & $8$  & 3.28 & \textbf{2.02} \\
 & $16$ & 3.29 & 1.97 \\
 & $32$ & \textbf{3.31} & 1.98 \\
 & $64$ & 3.30 & 1.96 \\
\midrule
\multirow{4}{*}{Top-$k$}
 & $8$  & 3.24 & 1.99 \\
 & $16$ & 3.28 & \textbf{2.02} \\
 & $32$ & \textbf{3.29} & 1.97 \\
 & $64$ & 3.26 & 1.95 \\
\bottomrule
\end{tabular}
\end{table}

\paragraph{Correction rank.}
Increasing the rank from $4$ to $32$ improves AAL from $3.20$ to $3.31$, followed by a slight decrease to $3.30$ at rank $64$. However, the highest E2E speedup is achieved at $r=8$ ($2.02\times$). Larger ranks provide only small additional gains in acceptance and do not improve overall training speedup.

\paragraph{Feedback size.}
With $k=8$, SPARK achieves an AAL of $3.24$ and an E2E speedup of $1.99\times$. Increasing $k$ to $16$ improves both metrics to $3.28$ and $2.02\times$. The highest AAL ($3.29$) occurs at $k=32$, but E2E speedup decreases to $1.97\times$; increasing $k$ further to $64$ reduces both metrics. These results suggest that using more feedback tokens does not necessarily yield better end-to-end efficiency.

Overall, $r=8$ and $k=16$ achieve the best E2E speedup among the tested settings while retaining AAL close to the maximum observed in each sweep. We therefore adopt this configuration as the default in our experiments.

\end{document}